\documentclass{article}

\usepackage[colorlinks,
           linkcolor=red,
           citecolor=blue,
           urlcolor=magenta,
           linktocpage,
           plainpages=false]{hyperref}

\PassOptionsToPackage{numbers, compress}{natbib}
\usepackage{graphicx} 
\usepackage{algorithm}
\usepackage{algorithmic}
\usepackage{amsthm}
\usepackage{amsmath}
\usepackage{amssymb}
\usepackage{bm}
\usepackage{bbm}
\usepackage{mathtools}
\usepackage{xspace}
\usepackage{wrapfig}
\usepackage{bookmark}

\usepackage[final]{nips_2016}

\usepackage[utf8]{inputenc} 
\usepackage[T1]{fontenc}    
\usepackage{hyperref}       
\usepackage{url}            
\usepackage{booktabs}       
\usepackage{amsfonts}       
\usepackage{nicefrac}       
\usepackage{microtype}      

\usepackage[disable]{todonotes}
\definecolor{babyblue}{rgb}{0.54, 0.81, 0.94}
\definecolor{citrine}{rgb}{0.89, 0.82, 0.04}
\definecolor{misocolor}{rgb}{0.16,0.27,0.86}

\providecommand{\medcap}{\bigcap}
\DeclareMathOperator*{\argmax}{arg\,max}

\newcommand{\ceil}[1]{\left\lceil#1\right\rceil}
\newcommand{\floor}[1]{\left\lfloor#1\right\rfloor}

\newtheorem{lemma}{Lemma}
\newtheorem{theorem}{Theorem}
\newtheorem{definition}{Definition}

\newtheorem{remark}{Remark}

\newcommand{\R}{\mathbb{R}}

\newcommand{\EE}[1]{\mathbb{E}\left[#1\right]}

\newcommand{\EEs}[2]{\mathbb{E}_{#1}\left[#2\right]}

\newcommand{\PP}[1]{\mathbb{P}\left[#1\right]}

\newcommand{\pa}[1]{\left(#1\right)}

\newcommand{\ac}[1]{\left\{#1\right\}}

\DeclarePairedDelimiter\abs{\lvert}{\rvert}%

\newcommand{\cG}{\mathcal{G}}

\newcommand{\cL}{\mathcal{L}}

\newcommand{\cO}{\mathcal{O}}
\newcommand{\tcO}{\widetilde{\cO}}

\newcommand{\cX}{\mathcal{X}}

\newcommand{\eps}{\varepsilon}
\renewcommand{\epsilon}{\varepsilon}
\renewcommand{\hat}{\widehat}

\newcommand{\nothere}[1]{}

\usepackage{xspace}

\newcommand{\UCB}{\texttt{UCB}\xspace}

\newcommand{\UCT}{\texttt{UCT}\xspace}

\newcommand{\olop}{\texttt{OLOP}\xspace}
\newcommand{\stopalgo}{\texttt{StOP}\xspace}
\newcommand{\metagrill}{\texttt{\textcolor[rgb]{0.5,0.2,0}{\textup{TrailBlazer}}}\xspace}

\newcommand{\maxn}{\texttt{\textup{\textcolor[rgb]{0.54, 0.1, 0.4}{MAX}}}\xspace}
\newcommand{\avgn}{\texttt{\textup{\textcolor[rgb]{0.54, 0.1, 0.4}{AVG}}}\xspace}

\usepackage{xifthen}

\newcommand{\Child}[1]{\mathcal{C}\left[#1\right]}
\newcommand{\Value}[1]{\mathcal{V}\left[#1\right]}
\newcommand{\Tree}{\mathcal{T}}
\newcommand{\K}{K}
\newcommand{\N}{N}

\newcommand{\reward}[1]{r_{#1}}
\newcommand{\trans}[1]{\tau_{#1}}
\newcommand{\result}{\mu}

\newcommand{\vfunc}{V}
\newcommand{\ncall}{n}
\newcommand{\nalg}{m}
\newcommand{\estimate}{\mu}
\newcommand{\conf}{\texttt{\textup{U}}}
\newcommand{\klocal}{k}
\newcommand{\elocal}{e}
\newcommand{\slocal}{s}
\newcommand{\tlocal}{t}
\newcommand{\qfun}{q}
\newcommand{\freevar}{t}
\newcommand{\sP}[1]{p\left(#1\right)}
\newcommand{\noset}{\mathcal{N}}

\let\oldeps\epsilon
\let\olddelta\delta
\let\oldgamma\gamma
\renewcommand{\epsilon}{\textcolor[rgb]{0.1,0.1,0.7}{\oldeps}}
\renewcommand{\eps}{\textcolor[rgb]{0.1,0.1,0.7}{\oldeps}}
\renewcommand{\delta}{\textcolor[rgb]{0.1,0.7,0.7}{\olddelta}}
\renewcommand{\gamma}{\textcolor[rgb]{1, 0.30, 0.04}{\oldgamma}}

\renewcommand{\ln}{\log}

\newcommand{\nocond}[1]{16\frac{\gamma^{(#1)/2}}{\gamma(1-\gamma)}}

\title{Blazing the trails before beating the path: Sample-efficient Monte-Carlo planning}

\author{Jean-Bastien Grill   \hspace{4em}  Michal Valko\\ SequeL team, INRIA Lille - Nord Europe, France\\ \texttt{\small jean-bastien.grill@inria.fr  \hspace{.5em}michal.valko@inria.fr}
\And \hspace{-3em}  R\'emi  Munos\\ \hspace{-3em} Google DeepMind, UK\thanks{on leave from SequeL team, INRIA Lille - Nord Europe, France}\\ \hspace{-3em}\texttt{\small munos@google.com}}

\begin{document}
\maketitle

\begin{abstract} 
You are a robot and you live in a Markov decision process (MDP) with a finite or an
infinite number of transitions from state-action to next states. You got brains and so
you \emph{plan} before you act. Luckily, your roboparents equipped you with a generative
model to do some \emph{Monte-Carlo planning}. The world is waiting for you and you
have no time to waste. You want your planning to be efficient. \emph{Sample-efficient}.
Indeed, you want to exploit the possible structure of the MDP by exploring only a
subset of states reachable by following near-optimal policies. You want guarantees
on sample complexity that depend on a measure of the quantity of near-optimal
states. You want something, that is an extension of Monte-Carlo sampling (for
estimating an expectation) to problems that alternate maximization (over actions)
and expectation (over next states). But you do not want to \textsc{StOP} with exponential
running time, you want something simple to implement and computationally
efficient. You want it all and you want it now. You want \textsc{TrailBlazer}.

\end{abstract} 

\section{Introduction}%
We consider the problem of sampling-based planning in a Markov decision process (MDP) when a {\em generative model} (oracle) is available. This approach, also called Monte-Carlo planning or Monte-Carlo tree search (see e.g., \cite{kocsis2006bandit}), has been popularized in the game of computer Go \citep{coulom2007efficient,gelly2006modifications,silver2016mastering} and shown impressive performance in many other high dimensional control and game problems~\citep{browne2012survey}. In the present paper, we provide a sample complexity analysis of a new algorithm called \metagrill.

Our assumption about the MDP is that we possess a generative model which can be called from any state-action pair to generate rewards and transition samples. Since making a call to this generative model has a cost, be it a numerical cost expressed in CPU time (in simulated environments) or a financial cost (in real domains), our goal is to \emph{use} this \emph{model} as \emph{parsimoniously} as possible.

Following \emph{dynamic programming}  \citep{bellman1957dynamic},  planning can be reduced to an approximation of the (optimal) value function, defined as the maximum of the expected sum of discounted rewards:
$\EE{\textstyle\sum_{t \geq 0} \gamma^t r_t}\!,$
where $\gamma\in[0,1)$ is a known \emph{discount factor}. Indeed, if an $\epsilon$-optimal approximation of the value function at any state-action pair is available, then the policy corresponding to selecting in each state the action with the highest approximated value will be $\cO\pa{\epsilon/\pa{1-\gamma}}$-optimal \citep{bertsekas1996neuro-dynamic}.

Consequently, in this paper, we focus  on a near-optimal approximation of the value function \emph{for a single given state} (or state-action pair). In order to assess the performance of our algorithm we measure its  {\em sample complexity} defined as the number of oracle calls, given that we guarantee its \emph{consistency}, i.e.,  that with probability at least $1-\delta$, \metagrill returns an $\epsilon$-approximation of the value function as required by the probably approximately correct (PAC) framework.

We use a \emph{tree representation} to represent the set of states that are reachable from any initial state. This tree alternates maximum (\maxn) nodes (corresponding to actions) and average (\avgn) nodes (corresponding to the random transition to next states). We assume the number $\K$ of actions is finite. However, the number $\N$ of possible next states is either \emph{finite} or \emph{infinite} (which may be the case when the state space is infinite), and we will report results in both the finite $\N$  and the infinite case. The root node of this planning tree represents the current state (or a state-action) of the MDP and its value is the maximum (over all policies defined at \maxn nodes) of the corresponding expected sum of discounted rewards. Notice that by using a tree representation, we do not use the property that some state of the MDP can be reached by  different paths (sequences of states-actions). Therefore, this state will be represented by different nodes in the tree. We could potentially merge such duplicates to form a graph instead. However, for simplicity, we choose not to merge these duplicates and keep a tree, which could make the planning problem harder.
To sum up, our goal is to return, with probability $1-\delta$, an $\eps$-accurate value of the root node of this planning tree while using as low number of calls to the oracle as possible. Our contribution is an algorithm called \metagrill whose sampling strategy depends on the specific structure of the MDP and for which we provide \emph{sample complexity} bounds in terms of a new \emph{problem-dependent measure of the quantity of near-optimal nodes}. 
Before describing  our contribution in more detail we first relate our setting to what has been around.

\subsection{Related work}
\citet{kocsis2006bandit} introduced the \UCT algorithm (upper-confidence bounds for trees). \UCT is efficient in computer Go \citep{coulom2007efficient,gelly2006modifications,silver2016mastering} and a number of other control and game problems~\citep{browne2012survey}. \UCT is based on generating trajectories by selecting in each \maxn node the action that has the highest upper-confidence bound, computed according to the \UCB algorithm of~\citet{auer2002finite}. \UCT converges asymptotically to the optimal solution, but its sample complexity can be worse than doubly-exponential 
in\footnote{
all guarantees in the form of $1/\epsilon^c$ are up to a poly-logarithmic multiplicative factor} $(1/\epsilon)$ for some MDPs \citep{munos2014from}. One reason for this is that the algorithm can expand very deeply the apparently best branches but may lack sufficient exploration, especially when a narrow optimal path is hidden in a suboptimal branch. As a result, this approach works well in some problems with a  specific structure but may be much worse than a uniform sampling in other problems.

On the other hand, a uniform planning approach is safe for all problems. \citet{kearns1999sparse} generate a sparse look-ahead tree based on expanding all \maxn nodes and sampling a finite number of children from \avgn nodes up to a fixed depth that depends on the desired accuracy $\epsilon$. Their sample complexity is\footnote{neglecting exponential dependence in $\gamma$} of the order of $(1/\epsilon)^{\log (1/\epsilon)}$, which is \emph{non-polynomial} in~$1/\eps$. This bound is better than that for \UCT in a worst-case sense. However, as their look-ahead tree is built in  a \emph{uniform} and {\em non-adaptive} way, this algorithm fails to benefit from a potentially favorable structure of the MDP. 

An improved version of this sparse-sampling algorithm by \citet{walsh2010integrating} cuts suboptimal branches in an adaptive way but unfortunately does not come with an improved bound and stays non-polynomial even in the simple Monte Carlo setting for which $\K=1$.

Although the sample complexity is certainly non-polynomial in the worst case, it \emph{can be polynomial} in some specific problems.
First, for the case of finite $\N$, the sample complexity is polynomial and \citet{szorenyi2014optimistic} show that a uniform sampling algorithm has complexity at most $(1/\epsilon)^{2+\log(\K\N)/(\log (1/\gamma))}$. Notice that the product $\K\N$ represents the \emph{branching factor} of the look-ahead planning tree. This bound could be improved for problems with specific reward structure or transition smoothness. In order to do this, we need to design non-uniform, \emph{adaptive} algorithm that captures the possible structure of the MDP when available, while making sure that in the worst case, we do not perform worse than a uniform sampling algorithm. 

The case of deterministic dynamics ($N=1$) and rewards considered by \citet{hren2008optimistic} has a complexity of order $(1/\epsilon)^{(\log \kappa)/(\log (1/\gamma))}$, where $\kappa\in [1,\K]$ is the branching factor of the subset of near-optimal nodes.\footnote{nodes that need to be considered in order to return a near-optimal approximation of the value at the root} The case of stochastic rewards has been considered by \citet{bubeck2010open} but with the difference that the goal was not to approximate the optimal value function but the value of the best {\em  open-loop} policy which consists in a sequence of actions independent of states. Their sample complexity is $(1/\epsilon)^{\max(2, (\log \kappa)/(\log 1/\gamma))}.$

In the case of general MDPs, \citet{busoniu2012optimistic} consider the case of a fully known model of the MDP. For any state-action, the model returns the expected reward and the set of all next states (assuming $\N$ is finite) with their corresponding transition probabilities. In that case, the complexity is $(1/\epsilon)^{\log \kappa/(\log (1/\gamma))}$, where $\kappa\in [0,\K\N]$ can again be interpreted as a branching factor of the subset of near-optimal nodes. These approaches use the \emph{optimism in the face of uncertainty} principle whose applications to planning 
have been studied by~\citet{munos2014from}.
\metagrill is different. It is \emph{not optimistic} by design: To avoid voracious demand for samples it does not balance the upper-confidence bounds of all possible actions.
This is crucial for polynomial sample complexity in the infinite case. The whole Section~\ref{sec:cogs} shines many rays of intuitive light on this single and powerful idea.

The work that is most related to ours is \stopalgo by~\citet{szorenyi2014optimistic} which considers the planning problem in MDPs with a generative model. 
Their complexity bound is of the order of $(1/\epsilon)^{2 + \log \kappa/(\log (1/\gamma))+o(1)}$, where $\kappa\in [0,\K\N]$ is a problem-dependent quantity. 
However, their $\kappa$ defined as $\lim_{\epsilon\rightarrow 0} \max(\kappa_1,\kappa_2)$ (in their Theorem 2) is somehow difficult to interpret as a measure of the quantity of \emph{near-optimal nodes}. Moreover, \stopalgo is not computationally efficient as it requires to identify the \emph{optimistic policy} which requires computing an upper bound on the value of {\em any} possible policy, whose number is exponential in the number of \maxn nodes, which itself is exponential in the planning horizon. Although they suggest (in their Appendix F) a computational improvement, this version is not analyzed. Finally, unlike in the present paper, \stopalgo does not consider the case $N=\infty$. 

\subsection{Our contributions}
Our main result is \metagrill, an algorithm with a bound on the number of samples required to return a high-probability $\epsilon$-approximation of the root node whether the number of next states $N$ is finite or infinite. The bounds use a problem-dependent quantity ($\kappa$ or $d$) that measures the quantity of near-optimal nodes. We now
summarize the results.

\textbf{\underline{Finite} number of next states} ($\N<\infty$): The sample complexity of \metagrill is of the order of\footnote{neglecting logarithmic terms in $\epsilon$ and $\delta$}
$( 1/\epsilon )^{\max(2,\log(N\kappa)/\log (1/\gamma)+o(1)) },$
where $\kappa\in[1,\K]$ is related to the branching factor of the set of near-optimal nodes (precisely defined later).

\textbf{\underline{Infinite} number of next states ($\N=\infty$)}: The complexity of \metagrill is 
$( 1/ \epsilon)^{2+d},$ where $d$ is a measure of the difficulty to identify the near-optimal nodes. Notice that $d$ can be finite even if the planning problem is very challenging.\footnote{since when $N=\infty$ the actual branching factor of the set of reachable nodes is infinite} 
We also state our contributions in specific settings in comparison to previous work.
\begin{itemize} 
\item For the case $N<\infty$, we improve over the best-known previous worst-case bound with an exponent of $1/\epsilon$) to $\max(2,\log(NK)/\log (1/\gamma))$ instead of $2+\log(NK)/\log (1/\gamma)$ reported by \citet{szorenyi2014optimistic}.
 \item For the case $N=\infty$, we identify properties of the MDP (when $d=0$) under which the sample complexity is of order of $1/\epsilon^2$. This is the case when there are non-vanishing action-gaps\footnote{defined as the difference in values of best and second-best actions} from any state along near-optimal policies or when the probability of transitioning to nodes with gap $\Delta$ is upper bounded by $\Delta^2$. 
 This complexity bound is as good as Monte-Carlo sampling and for this reason {\bf \metagrill is a natural extension of Monte-Carlo sampling} (where all nodes are \avgn) {\bf to stochastic control problems} (where \maxn and \avgn nodes alternate). Also, no previous algorithm reported a polynomial bound when $\N=\infty$. 
 \item In MDPs with deterministic transitions ($N=1$) but stochastic rewards our bound is $( 1/\epsilon)^{\max(2,\log \kappa/(\log 1/\gamma))}$ which is similar to the bound achieved by \citet{bubeck2010open} in a similar setting (open-loop policies).
\item In the evaluation case without control ($\K=1$) \metagrill behaves exactly as Monte-Carlo sampling (thus achieves a complexity of $1/\epsilon^2$), even in the case $N=\infty$. 
 \item Finally, \metagrill is easy to implement and is numerically efficient.
\end{itemize}





\section{The algorithm}

\paragraph{The objective}

We begin by defining our notations. We operate on a tree $\Tree$, where each node from the root down is alternatively either an average (\avgn) or a maximum (\maxn) node. 
For any node~$s$, let $\Child{s}$ be the set of its children. We consider only trees $\Tree$ for which the cardinality of $\Child{s}$ for any \maxn node~$s$ is bounded by $\K$.  In the paper, we consider trees with potentially unbounded number of children for an  \avgn node. As a special case, however, we will also 
consider finite trees. If the cardinality of $\Child{s}$ is also bounded for any \avgn node $s$ then we denote such bound as~$\N$.  Our objective is to provide performance guarantees in the case $\N = \infty$ and possibly tighter, $\N$-dependent guarantees in the case of $\N < \infty$.

Each \avgn node $s$ is associated with a random variable $\trans{s} \in \Child{s}$ and for any $s'\in\Child{s}$ we define \[ p\left(s'|s\right) = \PP{\trans{s} = s'}. \] 
Furthermore, for any node $s'$ and for any of its ancestors $s$, we define  $p\left(s'|s\right)$ as
the product of the transition probabilities between the two nodes.
Any \avgn node $s$ is also provided with a random variable $\reward{s}\in  [0,1]$. 
For any node $s$, we define the value function $\vfunc(s)$:  If $s$ is an \avgn node and $\gamma\in(0,1)$ the discount factor, then
\[\vfunc(s) = \EE{\reward{s}} + \gamma \sum_{s'\in\Child{s}} p(s'|s)\vfunc(s').\]
If $s$ is a \maxn node, $\vfunc(s)$ 
is defined as the maximum value among its children 
\[\vfunc(s) = \max_{s'\in\Child{s}} \vfunc(s').\]
The planner has an access to the generative model that can be called for any \avgn node $s$ to either get a reward $\reward{}$ or a transition $\trans{}$ which are two  independent random variables identically distributed as $\reward{s}$ and $\trans{s}$. 
\noindent 
Using the notation above, our goal is given the root $s_0$ of $\Tree$ to compute its value $\vfunc(s_0)$. More precisely, given any $\delta$ and $\epsilon$, the objective of the planner is  to output a value $\result$ such that 
\[ \PP{\abs{\result - \vfunc(s_0)} > \epsilon} \le \delta\]
and such that the number of calls to the generative model $\ncall\pa{\epsilon, \delta}$ is as low as possible.

\subsection{Description of the algorithm}

\metagrill constructs a planning tree which is, at any time, a finite subset of the 
potentially infinite tree $\Tree$. Only the nodes that were accessed are explicitly represented in memory.
Each node of $\Tree$ is a persistent object with its own memory, which it may update upon being called by its parent.
Therefore, any node can be potentially called several times 
(with different parameters) during the execution of \metagrill and it will reuse (some of) its stored samples.
In particular, after node~$s$ receives a call from its parent $s'$, node $s$ performs internal computation and 
potentially calls its children in order to return a real value to its parent $s'$. 

We show the pseudocode of \metagrill in Algorithm~\ref{alg:meta} along with the 
subroutines for \maxn nodes in Algorithm~\ref{alg:max} and \avgn nodes in Algorithm~\ref{alg:avg}. 
Whether a node is \maxn or \avgn,  it is always called with two parameters $\nalg$ and $\epsilon$.
The parameter $m$ essentially controls the requested variance of the output and parameter $\epsilon$ controls the requested maximum bias. 

\paragraph{\maxn nodes: }
A \maxn node $s$ keeps  a lower and an upper bound of its 
children values which with high probability simultaneously hold  at all time. It sequentially calls its children with different parameters 
in order to get more and more precise estimates on their values. 
Whenever the upper bound of one child becomes lower than the maximum lower bound, 
this child is discarded. This process can stop in two ways: 1) The set~$\cL$ of the remaining children shrunk enough such that there is only one child remaining in it. In this case, it calls this child with the same parameters $s$ was  called and use this output to return up as a result. 2) The precision we have on the value of the remaining children is high enough. In this case, it returns the highest estimate of the children in $\cL$. 

\paragraph{\avgn nodes: } Every \avgn node keeps a list of all the children that it already sampled and a reward estimate  $r \in \R$. Note that the list may contain
the same child multiple times. After receiving a call with parameters $(\nalg,\epsilon)$ it checks if $\epsilon \ge 1/(2(1-\gamma))$. If this condition is verified, then
it returns $1/(2(1-\gamma))$. If not, it considers the first $m$ sampled children (and potentially samples more children from the generative model if needed). For every child $s'$ in this list, the parent calls it with parameters $(k,\epsilon/\gamma)$, where $k$ is the number of times a transition toward this child was sampled. It returns $r + \gamma\mu$, where $\mu$  is the average of all the children estimates.

\begin{algorithm}[t]
\begin{algorithmic}
\STATE {\bf Input:}  $\delta$, $\epsilon$ 
\STATE {\bf Initialization:} 
 SampledNodes $\gets \emptyset$,
$r\gets0$, $\mu \gets 0$
\STATE {\bf Run:} 
\IF{$\epsilon \ge 1/(2(1-\gamma))$}
\STATE {\bf Output: $1/(2(1-\gamma))$} 
\ENDIF{}
\WHILE{ $\left| \text{SampledNodes} \right|  < \nalg$}
	\STATE sample a new next state and add it to SampledNodes
	\STATE sample a new reward and update $r$ as the average of all sampled rewards
\ENDWHILE{}
\STATE ActiveNodes $\gets$ SampledNodes$(1:m)$
\FOR{all unique nodes $s\in$ ActiveNodes}
	\STATE $k \gets$ number of occurrences of $s$ in ActiveNodes
	\STATE $\nu \gets$ call $s$ with parameters $(k, \epsilon/\gamma)$
	\STATE $\mu \gets \mu + \nu k/n$
\ENDFOR{}
\STATE {\bf Output:} $r + \gamma\mu$
\end{algorithmic}
\caption{\avgn node}
\label{alg:avg}
\end{algorithm}

\begin{algorithm}[H]
\begin{algorithmic}
\STATE {\bf Input:}  $\delta$, $\epsilon$ 
\STATE {\bf Initialization:} 
\STATE $\eta\gets \gamma^{1/\max(2,\ln(1/\epsilon))}$ 
\STATE $k_i  = 0, \text{ and } \mu_i  = 0, \forall i=1\dots K$ 
\STATE $\mathcal{L}  \gets\{ i=1\dots K \}$
\WHILE{$\left|\left\{i\in\mathcal{L} : U_i > \epsilon\right\}\right| > 1$}
	\STATE $l\gets\arg\min_i \left\{k_i\text{ : }\mu_i + 2U_i \ge \sup_j \left[\mu_j - 2U_j\right]\right\}$
	\STATE $k_l\gets k_l +1$
	\STATE $\conf_l \gets \frac{4}{(1-\eta)(1-\gamma)}\sqrt{\frac{\ln{(t/\delta)}}{k_l}}$
	\STATE $\mu_l\gets$ call  child $l$ with parameters $(k_l, \eta \max(U_l,\epsilon))$
	\STATE  $\mathcal{L}  \gets\left\{i\text{ : }\mu_i + 2U_i \ge \sup_j \left[\mu_j - 2U_j\right]\right\}$
\ENDWHILE{}
\IF{$\abs{\cL} = 1$}
\STATE   {\bf Output:} $\mu \gets$ call  child $l^\star$ with parameters $\pa{\nalg,\epsilon}$ where $l^\star = \argmax_{l\in\mathcal{L}} \mu_l$
\ELSE
\STATE  {\bf Output:} $\mu \gets \max_{i\in\cL} \mu_i$
\ENDIF
\end{algorithmic}
\caption{\maxn node}
\label{alg:max}
\end{algorithm}

\begin{algorithm}[H]
\begin{algorithmic}
\STATE {\bf Input:} 
 $\delta$ (global parameter), $\epsilon$ 
\STATE {\bf Set:} 
$\nalg\gets \ln(1/\delta)/((1-\gamma)^2\epsilon^2)$
\STATE {\bf Output:} $\result \gets $ call the root with parameters $(m,\epsilon/2)$
\end{algorithmic}
\caption{\metagrill}
\label{alg:meta}
\end{algorithm}


\section{Key ideas}\label{sec:cogs}
In this section, we explain the main ideas behind \metagrill. First we give some intuition about the \emph{two parameters} which measure the requested \emph{precision} of a call. The output estimate~$\mu$ of any call with parameters $(\epsilon, m)$ verifies the following property (conditioned on a high probability event)

\begin{align}
\forall\lambda\hspace{5mm}\EE{e^{\lambda\pa{\mu - v_s}}} \le \exp{\pa{ \epsilon\abs{\lambda} + \frac{\sigma^2\lambda^2}{2} }},
\quad\text{ with } \sigma^2 = \cfrac{1}{m(1-\gamma)^2}.
\end{align}
The $\epsilon$-factor  in the first term in the exponent is the upper bound on absolute value of the bias of our estimate~$\mu$
\[\abs{\EE{\mu}-V(s)}\leq\epsilon.\] 
 The second term in the exponent is a variance  term. 
The purpose of this property is to distinguish between variance and bias. This property is the core of 
Lemma~\ref{lemma:induc}, which is proved by induction\footnote{Having the proof by bottom-up induction opens a way to extend our result from the single-agent planning that we consider in the present to a two-agent planning with zero sum rewards.} on tree $\Tree$. Applying the Lemma~\ref{lemma:induc} to the root with the values of $\delta, \epsilon$, and $m$ given by Algorithm~\ref{alg:meta} immediately implies Theorem~\ref{thm:cons}. 

\subsection{Separate handling of bias and variance}

Instead of maintaining two parameters,  a common solution is to keep only a single parameter, representing the 
 width of a high probability confidence set. This confidence term typically sums the bias and a term proportional to the standard deviation. 
Notice, however, that \avgn node computes the average of \emph{independent} estimates. 
Therefore, while the bias just sums up,  it is the \emph{square} of the standard deviation which we sum to get the final variance of an \avgn node. 
Therefore, only a single confidence interval is not enough because summing them would lead to sum of a standard deviations instead of their squares.
This leads to a problem in the planning setting as we would like to average high-variance and low-bias estimates.
 In \metagrill, the nodes can perform high-variance and low-bias queries but can also query for both low-variance and low-bias. 
\metagrill treats these two types of queries differently. 

\subsection{Picking the champion paths}
We start with a simple illustration. First, imagine we want to compute an estimate of the expected value of some random variable $X \sim \cX$, such that $X \in [0,1]$ based on i.i.d.\@ samples $X_i \sim \cX$. 
In the second case,  we modify the problem such that instead of computing the estimate of $\EE{X}$ based on~$X_i$ we can only access independent but noisy values $Y_i \in [0,1]$, such that $\EE{Y_i} = \EE{X_i}$ and such that the upper bound on the variance of $Y_i$ and $X_i$ are the same.
Even though in the first problem we have the exact values of $X_i$, the guarantees we are able to provide are in general not better than the ones for the second. 

In our planning setting, this means that if an \avgn node needs $m$ independent samples from its child, it is wasteful (it would unnecessarily increase the sample complexity) for this child to perform more than $m$ independent samples. While the additional samples would increase the precision we have for the value of the child, they would not increase the precision we have on the value of the parent, because as in our simple illustration, the upper bound on the variance of the value of the parent stays the same.  This crucial step is done in the subroutine for the  \avgn node with
\[\text{ActiveNodes} \leftarrow \text{SampledNodes}(1:m).\]

Instead of considering all the transitions already sampled, it considers only the first $m$. While additional transitions were useful for some \maxn node parents to decide which action to pick, they should be discarded once this choice is made. Note that they can become useful again if an ancestor becomes unsure about which action to pick and needs more precision to make a choice. This treatment enables us to provide polynomial bounds on the sample complexity for some special cases even in the infinite case. 

\subsection{Tree-based algorithm}

The number of policies the planner can follow is exponential in the number of states. This leads to two major challenges. 
First, reducing the problem to a multi-arm bandits on the set of the policies would be disadvantageous. When a reward is collected from a state, all the policies which could reach that state are affected. Therefore, it is useful to share the information between the policies. 
The second challenge is computational as it is infeasible to keep all potential policies in the memory. 

These two problems immediately vanish with just how \metagrill is formulated. 
Contrary to \citealp{szorenyi2014optimistic}, we do not represent the policies explicitly or update
them simultaneously to share the information, but we store all the information directly in the planning tree we construct.
Indeed, by having all the nodes being independent entities that store their own information, we can share information between policies without having to enforce it explicitly. 

\subsection{Non-adaptive bias}
Whenever an \avgn node needs to furnish an estimate with bias $\epsilon$, it calls its children with parameters $(k,\epsilon)$ with $k$ being number of rounds that the child was sampled.  One could think that it would be more effective to use  a bias $\epsilon(k)$ which is a function of $k$. This would enable the algorithm to ask for a lower bias from the children that have been already sampled many times, and a higher bias from the children that have fewer samples. 
However in this case, the optimal function $\epsilon(k)$ depends on the difficulty of each sub-tree which is unknown. Moreover, in the infinite case the sample complexity is linear with $k$ and then a non-adaptive bias is optimal.


\section{\metagrill is good and cheap --- consistency and sample complexity} 
In this section, we start by our consistency result, stating that \metagrill outputs a correct value in a PAC (probably approximately correct) sense.  Later, we define a measure of the problem difficulty which we use to state our sample-complexity results. We remark that the following consistency result holds whether the state space is finite or infinite.
\begin{theorem}
\label{thm:cons}
For all $\epsilon$ and $\delta$, the output $\result_{\epsilon,\delta}$ of \metagrill called on the root $s_0$ with $(\epsilon, \delta)$  verifies
\[\PP{\abs{\result_{\epsilon,\delta}-\Value{s_0}} > \epsilon} < \delta.\]
\end{theorem}


\subsection{Definition of the problem difficulty}

We now define a measure of problem difficulty that we use to provide our sample complexity guarantees. We define a set of \textit{near-optimal} nodes such that exploring only this set is enough to compute an optimal policy. 
Let $s'$ be a \maxn node of tree $\Tree$. For any of its descendants $s$, let $c_{\rightarrow s}(s')\in\Child{s'}$ be the child of $s'$ in the path between $s'$ and $s$. For any \maxn node $s$, we define 
\[\Delta_{\rightarrow s}(s') = \max_{x\in\Child{s'}
}\Value{x} - \Value{c_{\rightarrow s}(s')}. \] 
$\Delta_{\rightarrow s}(s')$ is the difference of the sum of discounted rewards stating from $s'$ between an agent playing optimally and one playing first the action toward $s$ and then optimally.  
\begin{definition}[near-optimality]\label{def:near-opt}
We say that a node $s$ of depth $h$ is near-optimal, if for any even depth~$h'$,  
\[\Delta_{\rightarrow s}(s_{h'}) \le \nocond{h-h'}\]
with $s_{h'}$ the ancestor of $s$ of even depth  $h'$. 
Let $\noset_h$ be the set of all near-optimal nodes of depth $h$. 
\end{definition}

\begin{remark}

Notice that the subset of near-optimal nodes contains all required information to get the value of the root. In the case $N=\infty$, when $p(s|s')=0$ for all $s$ and $s'$, then our definition of near-optimality nodes leads to the smallest subset in a sense we specify in Appendix \ref{ss:choiceno}. 
We prove that with probability $1-\delta$, \metagrill only explores near-optimal nodes. Therefore, the size of the subset of near-optimal nodes directly reflects the sample complexity of \metagrill.

\end{remark}

In Appendix \ref{ss:choiceno}, we discuss the negatives of other potential definitions of near-optimality.

\subsection{Sample complexity in the finite case}

We first state our result where the set of the \avgn children nodes is \emph{finite} and bounded by $N$. 

\begin{definition}
We define $\kappa \in [1,K]$ as the smallest number such that 
\[\exists C\ \forall h, \quad \left|\noset_{2h}\right| \le CN^{h}\kappa^{h}.\]
\end{definition}

Notice that since the
 total number of nodes 
 of depth $2h$ is bounded 
 by $(KN)^h$, $\kappa$ is 
 upper-bounded by~$K$, the
  maximum number of \maxn's 
   children. However $\kappa$ can be as low as $1$ in cases when the set of near-optimal nodes is small. 
\begin{theorem}
\label{thm:ncalls}
There exists $C>0$ and $K$ such that for all $\epsilon>0$ and $\delta>0$, with probability $1-\delta$, the sample-complexity of \metagrill (the number of calls to the generative model before the algorithm terminates) is
\[\ncall(\epsilon,\delta) \le C(1/\epsilon)^{\max\pa{2, \frac{\log(N\kappa)}{\ln(1/\gamma)}+o(1)}}\pa{\ln(1/\delta)+\ln(1/\epsilon)}^{\alpha},\]
where $\alpha=5$ when $\log(N\kappa)/\ln(1/\gamma) \ge 2$ and $\alpha=3$ otherwise. 
\end{theorem}

This provides a problem-dependent sample-complexity bound, which already in the worst case ($\kappa=K$) improves over the best-known worst-case bound $\tcO\left((1/\epsilon)^{2+\ln(KN)/\ln(1/\gamma)}\right)$~\cite{szorenyi2014optimistic}. 
This bound gets better as $\kappa$ gets smaller and is minimal when $\kappa=1$. This is, for example, the case when the gap (see definition given in Equation~\ref{eq:gap}) at \maxn nodes is uniformly lower-bounded by some $\Delta>0$. In this case, this theorem provides a bound of order $(1/\eps)^{\max\pa{2,\ln(N)/\ln(1/\gamma)}}$. However, we will show in Remark~\ref{remark:lowerdelta} that we can further improve this bound to $(1/\eps)^2$.


\subsection{Sample complexity in the infinite case}


%
%

Since the previous bound depends on $N$, it does not apply to the infinite case with $N = \infty$. We now provide a sample complexity result in the case $N=\infty$. However, notice that when $N$ is bounded, then \emph{both} results apply. 

We first define gap $\Delta(s)$ for any \maxn node $s$  as the difference between the best and second best arm,
\begin{equation}\label{eq:gap}
\Delta(s) = \Value{i^{\star}} - \max_{i\in\Child{s}, i\neq i^{\star}}\Value{i}\quad\text{ with }
i^{\star} = \argmax_{i\in\Child{s}} \Value{i}.
\end{equation}

For any even integer $h$, we define a random variable $S^h$ taking values among \maxn nodes of depth $h$, in the following way. First, from every \avgn nodes from the root to nodes of depth $h$, we draw a single transition to one of its children according to the corresponding transition probabilities. This defines a subtree with $K^{h/2}$ nodes of depth $h$ and we choose $S^h$ to be one of them {\em uniformly at random}. Furthermore, for any even integer $h' < h$ we note $S^h_{h'}$ the \maxn node ancestor of $S^h$ of depth $h'$. 

\begin{definition}
We define $d\ge0$ as the smallest $d$ such that for all $\xi$ there exists $a>0$ for which for all even $h>0$,
\[\EE{K^{h/2}\mathbbm{1}
\pa{S^h\in\mathcal{N}_h}
\prod_{\substack{0\le h' < h \\ h' \equiv 0(\!\!\!\!\!\!\mod 2)}}\pa{\frac{\xi}{\gamma^{h-h'}}}^{\mathbbm{1}\pa{\Delta(S^h_{h'}) < \nocond{h-h'}}} } \le a\gamma^{-dh}\]
If no such $d$ exists, we set $d=\infty$.
\end{definition}

This definition of $d$ takes into account the size of the near-optimality set (just like $\kappa$) but unlike $\kappa$ it also takes into account the difficulty to identify the near-optimal paths.  

Intuitively, the expected number of oracle calls performed by a given \avgn node $s$ is proportional to: ($1/\eps^2$) $\times$ (the product of the inverted squared gaps of the set of \maxn nodes in the path from the root to $s$) $\times$ (the probability of reaching $s$ by following a policy which always tries to reach $s$). 

Therefore, a near-optimal path with a larger number of small \maxn node gaps can be considered \emph{difficult}. By assigning a larger weight to difficult nodes, we are able to give a better characterization of the actual complexity of the problem and provide polynomial guarantees on the sample complexity for $N=\infty$ when $d$ is finite. 


%

\begin{theorem}\label{thm:ninfinite} 
If $d$ is finite then there exists $C>0$ such that for all $\epsilon>0$ and $\delta>0$, the expected sample complexity of \metagrill satisfies
\[{\EE{\ncall(\epsilon,\delta)} \le C\frac{\pa{\ln(1/\delta)+\ln(1/\epsilon)}^3}{\epsilon^{2+d}}}\cdot\]
\end{theorem}

Note that this result holds in expectation only, contrary to Theorem~\ref{thm:ncalls} which holds in high probability. 

We now give an example for which $d=0$, followed by a special case of it.
\begin{lemma}\label{lemma:b}
If there exists $c>0$ and $b>2$ such that for any near-optimal \avgn node $s$, 
\[\PP{\Delta\pa{\trans{s}} \le x} \le c x^b,\]
where the random variable $\trans{s}$ is a successor state from $s$ drawn from the MDP's transition probabilities, then $d=0$ and consequently the sample complexity is of order $1/\eps^2$. 
\end{lemma}

\begin{remark}\label{remark:lowerdelta}
If there exists $\Delta_{\min}$ such that for any near-optimal \maxn node $s$, $\Delta(s) \ge \Delta_{\min}$ then $d=0$ and the sample complexity is of order $1/\eps^2$. 
Indeed, in this case as $\PP{\Delta_s \le x} \le \pa{x/\Delta_{\min}}^b$ for any $b>2$ for which $d=0$ by Lemma~\ref{lemma:b}.
\end{remark}

\section{Discussion}

\subsection{On the choice of the near-optimality definition}
\label{ss:choiceno}
In this part, we discuss the possible alternative choices to the near-optimality set $\noset_h$ and explain the choice we made 
in Definition~\ref{def:near-opt}. From this definition, the knowledge of the rewards of the near-optimal nodes is enough to compute an optimal policy. 
Therefore, whatever the rewards associated with the non-near-optimal nodes are, they do not change the optimal policy. 
A good \emph{adaptive} algorithm, would not exploring the whole tree, but would with high probability only explore a set not significantly larger than a set $\noset_h$ of near-optimal nodes. 
There are other possible definitions of the near-optimality set satisfying this desired property that we could possible consider for the definition of $\noset_h$. 
Ideally, this set would be as small as possible.
\paragraph{Alternative definition 1}
An idea could be to consider for any node $s$ and an ancestor $s'$, the value of the policy starting from $s'$ choosing the action leading to $s$ every time it can
and playing optimally otherwise.
Let $V_s(s')$ be this value. 
A natural approach would be to define $\noset_h$ only based on $V_s(s_0)$ with $s_0$ being the root. 
To ensure that exploring the near-optimal nodes is enough to compute the optimal policy we need any near-optimal nodes $s$ of depth $h$ to verify
\[ \pa{V(s_0) - V_s(s_0)}p(s | s_0) \le \cfrac{\gamma^h}{1-\gamma}\cdot\]
When there is no \avgn node, $p(s) = 1$ and this definition coincides with the near-optimal set defined in \olop~\citep{bubeck2010open}. Nevertheless, in our case, $p(s)$ can be low and even 0 in the infinite case. When $p(s| s_0) = 0$ for all $s$, any node would be near-optimal which is bad because 
then the near-optimality set is large and the sample complexity of the resulting algorithm suffers.
\paragraph{Alternative definition 2}
There is a smarter way to define a near-optimality set. We could define it as all the nodes $s$ of depth $h$ such that
\[\forall h'\le h, \quad \sum_{i=h'}^h p(s_{i}|s_{h'})\gamma^i\left(V\left(s_{i}\right) - V(c(s_{i},s)\right) \le \frac{\gamma^{h}}{1-\gamma}\cdot\]
Taking $h' = 0$, we recover the alternative definition 1. Alternative definition 2 defines a smaller near-optimality set, as the condition needs to be verified for all $h'$ and not just for $h'=0$. This definition would also lead to a smaller near-optimality set than our Definition~\ref{def:near-opt}. Indeed, in Definition~\ref{def:near-opt} we consider only the first term of the sum above. This enables the algorithm to compute its strategy from a node $s$ only based on the value of its descendants. This
is an ingredient leading to better computational efficiency of \metagrill. Indeed, the computational complexity of our algorithm is linear in the number of calls to the generative model. 
In addition, when the transition probabilities are low, both definitions are close and coincide in the infinite case, when the probability transitions
are continuous and therefore $p(s|s')=0$.  

\subsection{Comparison of the finite and the infinite analysis}

The difference between our finite and the infinite analysis lies in how do we bound the probability of reaching at least once the node $s$ from $s'$ given we received $m$ samples from $s'$. This probability is equal to $1-(1-p(s|s'))^m$. In Figure~\ref{fig:graph}, we show the probability function of $m$ and the bounds we use for this function. For the infinite analysis, we use the linear approximation at $m=0$ as an upper bound. This bound is tight when $\ln(1-p(s|s')) \ll m$ which happens in the infinite case. For the finite case we bound this quantity by 1. 

\begin{figure}[H]
\begin{center}
\vspace{1em}
\includegraphics[width=0.6\columnwidth]{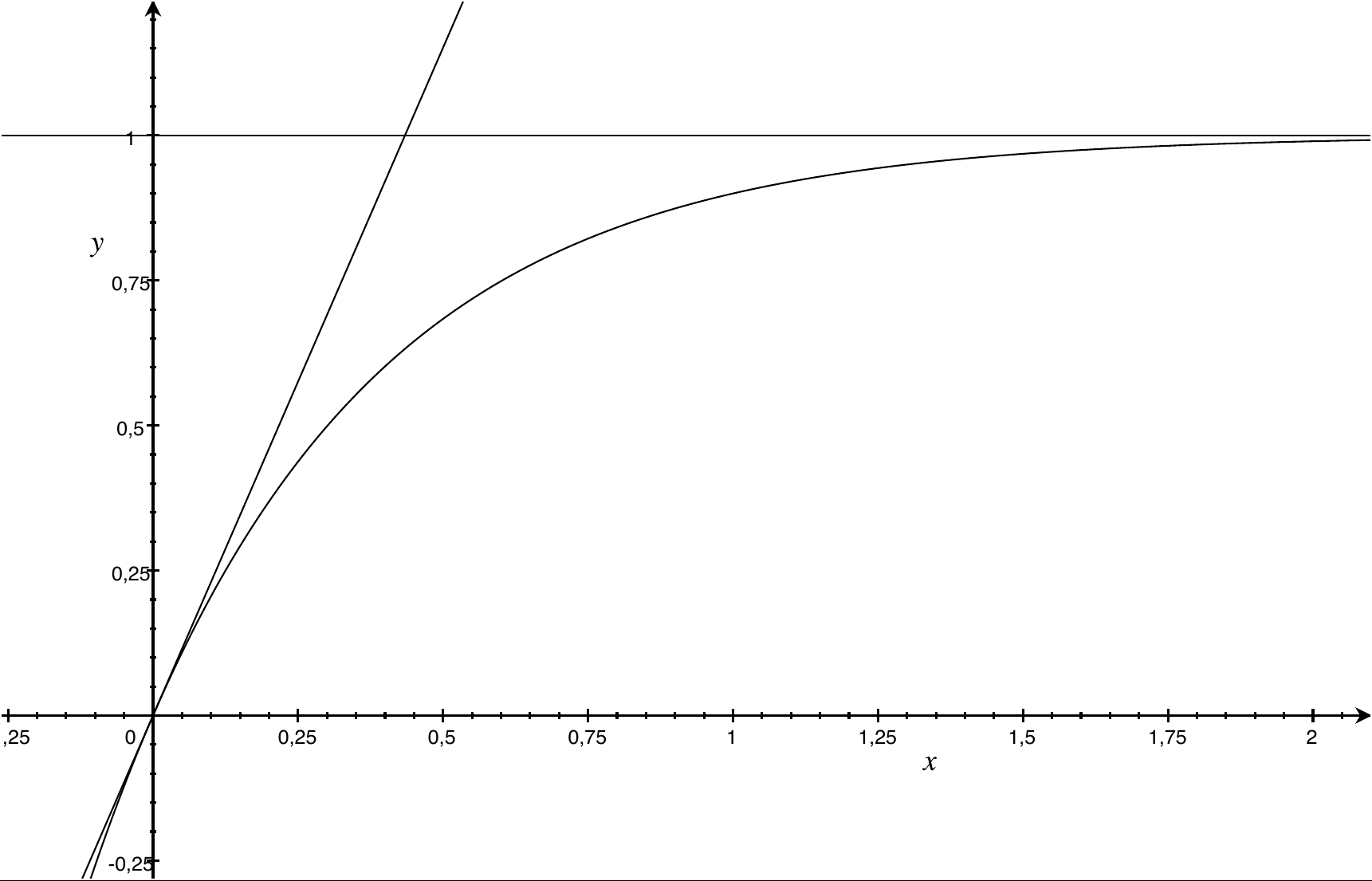}
\caption{Infinite vs.\,finite analysis}
\label{fig:graph}
\end{center}
\end{figure}

\clearpage

{\small
\paragraph{\small Acknowledgements}
\label{sec:Acknowledgements}
The research presented in this paper was supported by French Ministry of
Higher Education and Research, Nord-Pas-de-Calais Regional Council, a doctoral grant
of \'{E}cole Normale Sup\'erieure in Paris, Inria and Carnegie Mellon University associated-team project EduBand, and French National Research Agency projects ExTra-Learn (n.ANR-14-CE24-0010-01) and BoB (n.ANR-16-CE23-0003) }

\clearpage
\appendix
\section{Additional notation for the analysis}
In this part, we define new quantities to be used in the other part of the appendix.
For all \maxn nodes $i$ and time $\tlocal$, we denote $\mu^{s,t}_i$ the estimate of child~$i$ at time~$t$ and $(\epsilon^{s,t}_i, k^{s,t}_i)$ the parameters that have been use to get $\mu^{s,t}_i$. We then define 
\[\conf_i^{\slocal, \tlocal} = \cfrac{4}{1-\gamma}\sqrt{\cfrac{\ln(t/\delta)}{\klocal_i^{\slocal, \tlocal}}} \cdot\]

\section{Maximum depth of the search}
Next lemma ensures that \metagrill terminates and bounds the maximum depth it reaches.  
\begin{lemma}
\label{lemma:finitetree}
For all $\nalg$ and $\epsilon>0$, a call to the root with parameters $\pa{\nalg,\epsilon}$ never results in calls to a node of depth $h>h_{\text{max}}$ with 
\[h_{\text{max}} = 2\ceil{\frac{\ln\pa{1/\epsilon}+\ln\pa{1/(1-\gamma)}}{\ln\pa{\eta/\gamma}}}+2.\]
\end{lemma}
\begin{proof}
A \maxn node called with parameters $(\nalg,\epsilon)$ only performs calls with parameter $\epsilon'\ge\eta\epsilon$. An \avgn node called with parameters $(\nalg,\epsilon)$ only perform calls with parameter $\epsilon'\ge\epsilon/\gamma$. 
Since \maxn and \avgn nodes alternate, we get from an immediate induction that if the root is called with $\epsilon$, then for any $h>0$ any node at depth $2h$ is called with parameter $\epsilon' \ge \epsilon\eta^h/\gamma^h$ and therefore every node at depth $h$ is called with parameter
\[\epsilon' \ge \cfrac{\epsilon}{(\eta/\gamma)^{\floor{h/2}}}\cdot\]
For $h > h_{\text{max}} - 2$ we have that
\[\epsilon' \ge \cfrac{\epsilon}{(\eta/\gamma)^{\floor{(h_{\text{max}}-2)/2}}} \ge \cfrac{1}{1-\gamma}\cdot\]
By the algorithm definition, if $h$ is an \avgn node it returns 0 and does not perform any call. On the other hand, if $h$ is a \maxn node, it calls the \avgn nodes of next depth and these nodes return 0 without performing any call. In any case, no node of depth $h$ is ever called.  
\end{proof}


\begin{lemma}
\label{lemma:conf}
For any set $\mathcal{J}$ of node, if all of them verify the two property of Lemma~\ref{lemma:induc} then
\[\PP{\bigcup_t\bigcup_{i\in\mathcal{J}} B_{i,s}(t)} \le \delta/t\]
with 
\[B_{i,s}(t) = \ac{\abs{\mu^{s,t}_i - v_i} \ge \epsilon^{s,t}_i + U^{s,t}_i},\]
\end{lemma}

\begin{proof}
If for some $\alpha>0,\beta, M$, a random variable $X$ verifies that
\begin{align*}
\EE{e^{\lambda X}} 
&\le M\exp{\pa{\abs{\lambda}\alpha + \cfrac{\beta^2\lambda^2}{2}}}, 
\end{align*}
then using Markov inequality for $u>\alpha$ and setting $\lambda = \frac{u-\alpha}{\beta^2}$, we get 
\begin{align*}
\PP{X \ge u} 
\le \cfrac{\EE{e^{\lambda X}}}{e^{\lambda u}}
\le M\exp{\pa{\abs{\lambda}\alpha-\lambda u + \cfrac{\beta^2\lambda^2}{2}}}
\le M\exp{\pa{-\cfrac{(u-\alpha)^2}{2\beta^2}}} \cdot
\end{align*}
Similarly, for $u<0$
\begin{align*}
\PP{X \le u} \le M\exp{\pa{-\cfrac{(u+\alpha)^2}{2\beta^2}}}\cdot
\end{align*}
Taking a union bound,  we have that for all $u>0$
\begin{align*}
\PP{\abs{X} \ge u} 
&\le 2M\exp{\pa{-\cfrac{(u-\alpha)^2}{2\beta^2}}}\cdot
\end{align*}
We define a set of auxiliary events
\[B_{i,s}(t) = \ac{\abs{\mu^{s,t}_i - v_i} \ge \epsilon^{s,t}_i + U^{s,t}_i},\]
and also define the event for which we prove the properties $A$ and $B$
\[A_s(t) = \medcap_iA_{i,s}(t)\]
Notice that for all children $i,j$, $A_{i,s}(t)$ is independent from $\mu^{s,t}_j$ and $A_{j,s}(t)$. It follows that for any child $i$ and any time~$t$,
\begin{align*}
\EE{e^{\lambda(\mu^{s,t}_i - v_i)}\bigg|A_{i,s}(t)}
\le \cfrac{1}{(1-\delta/t)^{C_t(s)}}\exp\pa{\abs{\lambda}\epsilon^{s,t}_i + \cfrac{\lambda^2}{2(1-\gamma)^2k^{s,t}_i}}
\end{align*} 
We combine this inequality with the Markov inequality to get
\begin{align*}
\PP{B_{i,s}(t)\bigg|A_s(t)} 
&\le \cfrac{2}{(1-\delta/t)^{C_t(s)}}\exp{\pa{-\cfrac{\frac{1}{(1-\gamma)^2}\frac{16\ln(t/\delta)}{k^{s,t}_i}}{\frac{2}{(1-\gamma)^2k^{s,t}_i}}}}&\\
&\le \cfrac{2}{(1-\delta/t)^{t}}\exp\pa{-4\ln(t/\delta)-4\ln(t/\delta)}&\text{ (we bound $C_t(s)$ by $t$)}\\
&\le 8\exp\pa{-\ln(t^4/\delta) - 4\ln(2)}&\text{ (we use that $t\ge2$)}\\
&\le \delta/\pa{2t^4}.
\end{align*}
We now define
\[B_{s}(t) = \medcap_{i,t}B_{i,s}(t)^c,\]
and use a union bound over all visited children $i$ and all times $t$ to get
\begin{align*}
\PP{B_{s}(t)^c\big|A_s(t)}
&\le \bigcup_{i,t}B_{i,s}(t) &\\
&\le \sum_t\sum_i\delta/\pa{2t^4}&\\
&\le \sum_t\delta/\pa{2t^3}&\text{ (there are at most $t$ visited children)}\\
&\le 2\delta\pi^2/(6t)&\\
&\le \delta/t&
\end{align*}
\end{proof}
%
%
%
\section{Consistency} 
\begin{lemma}\label{lemma:induc}
Let $C_t(s)$ be the number of explored nodes at time $t$ which are descendants of $s$. 
For all nodes $s$, time $t$, $\epsilon>0$, and integer $m$,
let $\estimate^{s,t}$  be the output of a call to the node $s$ at time $t$ with parameters $(m,\epsilon)$. There exists an event $A_s(t)$ such that 
\begin{flalign}
& \quad \textbf{Property A:} \quad \PP{A_s(t)}\ge 1-\delta\cfrac{C_t(s)}{C_t(s_0)} &
\end{flalign}
\begin{flalign}
& \quad \textbf{Property B:} \quad \forall\lambda,\EE{e^{\lambda\pa{\estimate^{\slocal, \tlocal} - v_s}}\bigg|A_s(t)} \le \cfrac{1}{(1-\delta/t)^{C_t(s)}}\exp{\pa{\epsilon\abs{\lambda} + \cfrac{\sigma^2\lambda^2}{2}}} &
\end{flalign}
\end{lemma}

\begin{proof}
We prove the lemma by a bottom-up induction on the tree. We know the tree is bounded, since its size is trivially bounded by $t$. 
We start proving that properties $A$ and $B$ hold for all the leaves and then for any node $s$ we assume it holds for its children and prove it for $s$. 

\paragraph{Leaf node:} Let $s$ be a leaf of the tree $\Tree_t$. The only case in which a node $s$ does not perform a call to another deeper node is when it is an \avgn node which has been called with parameter $\epsilon \ge \tfrac{1}{2(1-\gamma)}\cdot$ In this case the output $0$ is deterministic and therefore the two first properties hold by choosing $A_{s,t} = \Omega$. Furthermore,  $\tfrac{1}{2(1-\gamma)}$ is an upper bound on $\abs{\mu^{s,t}-v_s}$ which implies that the second property is also verified for $s$. 

\paragraph{\maxn node:} Let $s$ be a maximum node. 
 We define 
\[\hat{i} = \argmax_i\mu^{s,t}_i\quad \text{and} \quad i^* = \argmax_iV(i).\]
We also define the random variable $\mathcal{L}^{s,t}$ as the set of the potentially good children as defined in the algorithm. Note that $\cL^{s,t}$ is a random variable. 
\[ \mathcal{L}^{s,t} = \ac{i\text{ : }\estimate_i^{s,t} + 2\epsilon \ge \sup_j \left[\mu_j^{s,t} - 2\epsilon\right]}\]
We also define $\mathcal{G}$ (which is \emph{not} a random variable) as follows
\[ \mathcal{G}_s = \ac{i\text{ : }v_i + \epsilon \ge \max_iv_i - \epsilon}.\]
We use the definition of $B_{i,s}(t)$ and $B_s(t)$ of Lemma~\ref{lemma:conf} with $\mathcal{J} = \Child{s}$
From now, we assume that $B_{s}(t)^c$ holds, which means that all estimates are at all time within their confidence intervals. 
At the end of the while loop, any child $i$ in $\cL^{s,t}$ verifies
\[\epsilon^{s,t}_i \le \eta\epsilon\quad\text{and}\quad U^{s,t}_i \le (1-\eta)\epsilon.\]
Therefore, if $B_{s}(t)^c$ holds we have for any child $i\in\cL^{s,t}$
\[\abs{\mu^{s,t}_i - V(i)} \le \epsilon\]
We now distinguish two cases depending on the cardinality of $\cG_s$. 
\paragraph{First case, $\abs{\cG_s} = 1$}.
For all $i\in\cL^{s,t}$
\begin{align*}
\mu^{s,t}_i &\le V(i) + \epsilon&\text{ (because $B_{s}(t)^c$ holds)}\\
&\le V(i^*) - \epsilon&\text{ (by definition of $\cG_s$)}\\
&\le \mu^{s,t}_{i^*}&\text{ (because $B_{s}(t)^c$ holds)}
\end{align*}
If $\abs{\cL^{s,t}} = 1$, then $\cL^{s,t} = \ac{i^*}$, and consequently $i^*$ is called with parameters $(m, \epsilon)$. 
On the other hand, if $\abs{\cL^{s,t}} > 1$ then $i^*$ is called with parameters $(m', \epsilon')$ with 
\[\epsilon' \le \eta\epsilon \le\epsilon   \quad\text{and}\quad   m' \ge \max\pa{k^{s,t}_{i^*},m} \ge m.\]
Therefore on the event $B_s(t)$, $\mu^{s,t}_i$ is the output of a call to $i^*$ with parameters $(m',\epsilon')$ with $m'\ge m$ and $\epsilon'\le\epsilon$. 
Furthermore, we have
\begin{align*}
&\EE{e^{\lambda\pa{\mu^{s,t}_{i^*} - v_s}}|A_s(t)} \\
&\hspace{5mm}=\EE{e^{\lambda\pa{\mu^{s,t}_{i^*} - v_s}}|A_s(t)\cap B_s(t)}\PP{B_s(t)|A_s(t)} 
+ \EE{e^{\lambda\pa{\mu^{s,t}_{i^*} - v_s}}|A_s(t)\cap B_s(t)^c}\PP{B_s(t)^c|A_s(t)} \\
&\hspace{5mm}\ge \EE{e^{\lambda\pa{\mu^{s,t}_{i^*} - v_s}}|A_s(t)\cap B_s(t)}\pa{1-\delta/t}
\end{align*}
Now we use that $1+C_t(i^*) \le 1+\sum_iC_t(i)  = C_t(s)$ with the induction assumption on $i^*$ to get
\begin{align*}
\EE{e^{\lambda\pa{\mu^{s,t} - v_s}}|A_s(t)\cap B_s(t)}
&\le \cfrac{1}{1-\delta/t}\EE{e^{\lambda\pa{\mu^{s,t}_{i^*} - v_s}}|A_s(t)}\\
&\le \cfrac{1}{\pa{1-\delta/t}^{C_t(i^*)+1}}\exp\pa{\abs{\lambda}\epsilon + \cfrac{\lambda^2}{2(1-\gamma)m}}\\
&\le \cfrac{1}{\pa{1-\delta/t}^{t}}\exp\pa{\abs{\lambda}\epsilon + \cfrac{\lambda^2}{2(1-\gamma)m}}.
\end{align*}

\paragraph{Second case, $\abs{\cG_s} > 1$}
When $B_{s}(t)^c$ holds, we have that
\[\cG_s \subset \cL^{s,t}.\]
Indeed, for all $i\in\cG_s$,
\begin{align*}
\mu^{s,t}_i &\ge V(i) - \epsilon&\text{ (because $B_{s}(t)^c$ holds)}\\
&\ge V(i^*)-3\epsilon& \text{ (because $i\in\cG_s$)}\\
&\ge\mu^{s,t}_{i^*} - 4\epsilon &\text{ (because $B_{s}(t)^c$ holds)}
\end{align*}
As a result $\abs{\cL^{s,t}} > \abs{\cG_s} > 1$. The output is then the maximum of the estimates in $\cL^{s,t}$. The best estimate $\mu^{s,t}_{\hat{i}} $ in $\cL^{s,t}$ verifies 
\begin{align*}
\mu^{s,t}_{\hat{i}} 
&\ge \mu^{s,t}_{i^*}&\text{ (by definition of $\hat{i}$)}\\
&\ge V(i^*) - \epsilon&\text{ (because $B_{s}(t)^c$ holds)}.
\end{align*}
It also verifies that
\begin{align*}
\mu^{s,t}_{\hat{i}} 
&\le V(\hat{i}) + \epsilon&\text{ (because $B_{s}(t)^c$ holds)}\\
&\ge V(i^*) + \epsilon&\text{ (by definition of $i^*$)}.
\end{align*}
On event $B_s(t)^c$, when $\abs{\cG_s} > 1$, then
\[\abs{\mu^{s,t} - V(s)} \le \epsilon.\]
From this we get
\begin{align*}
\EE{e^{\lambda\pa{\mu^{s,t} - v_s}}|A_s(t)\cap B_s(t)}
&\le \exp\pa{\abs{\lambda}\epsilon}\\
&\le \exp\pa{\abs{\lambda}\epsilon + \cfrac{\lambda^2}{2(1-\gamma)m}}\\
&\le \cfrac{1}{\pa{1-\delta/t}^{t}}\exp\pa{\abs{\lambda}\epsilon + \cfrac{\lambda^2}{2(1-\gamma)m}}
\end{align*}
Therefore, in both cases, whether $\abs{\cG_s} = 1$ or $\abs{\cG_s} > 1$, we proved that
\[\EE{e^{\lambda\pa{\mu^{s,t} - v_s}}|A_s(t)\cap B_s(t)}\le \cfrac{1}{\pa{1-\delta/t}^{t}}\exp\pa{\abs{\lambda}\epsilon + \cfrac{\lambda^2}{2(1-\gamma)m}}.\]
Finally, we lower bound the probability of $A_s(t)\cap B_s(t)$, using $C_t(s) = 1 + \sum_iC_t(i)$.
\begin{align*}
\PP{\pa{A_s(t)\cap B_s(t)}^c}
&= \PP{A_s(t)^c\cup B_s(t)^c}\\
&= \PP{A_s(t)^c\cup (B_s(t)^c\cap A_s(t)) \cup (B_s(t)^c\cap A_s(t)^c)}\\
&= \PP{A_s(t)^c\cup (B_s(t)^c\cap A_s(t))}\\
&\le \PP{A_s(t)^c} + \PP{B_s(t)^c\cap A_s(t)}\\
&\le \PP{A_s(t)^c} + \PP{B_s(t)^c\big| A_s(t)}\\
&\le \sum_i \delta\frac{C_t(i)}{C_t(s_0)} + \delta/t 
\le \delta\frac{C_t(s) - 1}{C_t(s_0)} + \delta/C_t(s) 
\le \delta\frac{C_t(s)}{C_t(s_0)}
\end{align*}

\paragraph{\avgn node:} Let $s$ be an average node. 
We define $A_s(t) = \medcap_i A_{s,i}(t)$. By a union bound
\begin{align*}
\PP{A_s(t)} \ge 1 - \sum_i \PP{A_{s,i}(t)^c}
\ge 1 - \delta\frac{C_t(s) - 1}{C_t(s_0)}
\ge 1 - \delta\frac{C_t(s)}{C_t(s_0)}.
\end{align*}
Note that the output is the average of independent estimates. We denote by $k_i$, the number of times the transition to child $i$ has been made and $\mu_i$ the estimate of $i$. We have that
\begin{align*}
\EE{e^{\lambda\pa{\frac{1}{m}\sum_i k_i\mu_i}}\bigg|A_s(t)}
&=\prod_i\EE{e^{\lambda\frac{k_i}{m}\mu_i}\bigg|A_s(t)}\\
&=\prod_i\EE{e^{\lambda\frac{k_i}{m}\mu_i}\bigg|A_{s,i}(t)}\\
&\le \prod_i\cfrac{1}{(1-\delta/t)^{C_t(i)}}\exp\pa{\abs{\lambda}\frac{k_i}{m}\epsilon + \cfrac{\lambda^2k_i^2}{2(1-\gamma)m^2k_i}}\\
&\le \cfrac{1}{(1-\delta/t)^{\sum_iC_t(i)}}\exp\pa{\abs{\lambda}\sum_i\frac{k_i}{m}\epsilon + \sum_i\cfrac{\lambda^2k_i}{2(1-\gamma)m^2}}\\
&\le \cfrac{1}{(1-\delta/t)^{C_t(s)}}\exp\pa{\abs{\lambda}\epsilon + \cfrac{\lambda^2}{2(1-\gamma)m}}.
\end{align*}
 \end{proof}

\subsection{Proof of Theorem~\ref{thm:cons}}

The theorem is a direct application of Lemma~\ref{lemma:conf} in combination with Lemma~\ref{lemma:induc},  taking $\mathcal{J} = \ac{s_0}$. 

\section{Sample complexity}
\begin{theorem}
The number of calls $\tlocal$ to the model verifies with probability $1-\delta$ 

\[\EE{T} \le \cfrac{4\ln\pa{1/\epsilon}}{\pa{1-\gamma}^{2\lambda+1}} \cfrac{1}{\epsilon^2}\pa{\cfrac{16\ln\pa{\nalg/\delta}}{\epsilon^2}}^\lambda\cdot \]

\end{theorem}

\begin{proof}
By Lemma~\ref{lemma:conf} we consider the event that every output $\estimate$ is within its confidence interval at all time.  Any node which is clearly sub-optimal is not explored by the algorithm on this event. 
We are going to the maximum $m$ parameter that a near-optimal node of depth $h$ is called with. This is a bound on the number of called to the model this node has performed. 

By Lemma~\ref{lemma:finitetree}, no node of depth $h > h_{\text{max}}$ is ever reached by the algorithm. 
Let $h,h',H$ be three non negative integers such that $h\le h_{\text{max}}$, $h'\le h$ and $H<h$. Furthermore, let $\slocal$ be a node of depth $h$ and $\slocal_{h'}$ be the ancestor of $\slocal$ of depth $h'$, so that $\slocal_{h} = s$ and $\slocal_{0}$ is the root. We also define $\qfun\pa{\slocal,h',H}$ as the number of nodes on the path from $\slocal_{h-H}$ to $\slocal_{h'}$ which are $\slocal$-undecidable such that $\qfun\pa{\slocal,0,H} = \qfun\pa{\slocal,H}$. For any node $\slocal$ we note $R_\slocal$ the random variable of the maximal parameter $m$ $s$ has been called with. 
We we would like to prove by induction on $h'$ that for any $h',H$
\[ \EE{R_\slocal} = \EE{R_{\slocal_{h'}}}\sP{\slocal|{\slocal_{h'}}}\pa{\cfrac{18\ln(t/\delta)}{(1-\eta)^2(1-\gamma)^2\epsilon^2}}^{\qfun\pa{\slocal,h',H}}\pa{\cfrac{18\ln(t/\delta)}{(1-\eta)^2\gamma^{2H}}}^{H}.\] 
We first fix $H=0$. 
For $h'=h$, this is trivial as $\qfun\pa{\slocal,h,0} = 0$ and $\sP{\slocal|\slocal} = 1$. 
Now let us assume that the property holds for some $h'+1\le h_{\text{max}} +1$ and let us prove it for $h$. There are two cases, either $\slocal_{h'}$ is an \avgn node or $\slocal_{h'}$ is a \maxn node. 
If $\slocal_{h'}$ is an \avgn node then
\[\EE{R_{\slocal_{h'+1}}} = \EE{R_{\slocal'\pa{h'}}}\sP{\slocal_{h'+1}|\slocal_{h'}}\]
Then we can use our induction assumption to get
\[\EE{R_\slocal} = \EE{R_{\slocal_{h'+1}}}\sP{\slocal|{\slocal_{h'+1}}}\pa{\cfrac{18\ln(t/\delta)}{(1-\eta)^2(1-\gamma)^2\epsilon^2}}^{\qfun\pa{\slocal,h'+1,0}}.\]
Furthermore, since $\slocal_{h'}$ is an \avgn node, $\qfun\pa{\slocal,h'+1,0} = \qfun\pa{\slocal,h',0}$ and 
\[\sP{\slocal|\slocal_{h'+1}}\sP{\slocal_{h'+1}|\slocal_{h}} = \sP{\slocal|\slocal_{h}}.\] 
Combining the previous equalities, we get the following
\[\EE{R_\slocal} = \EE{R_{\slocal_{h'}}}\sP{\slocal|{\slocal_{h'}}}\pa{\cfrac{18\ln(t/\delta)}{(1-\eta)^2(1-\gamma)^2\epsilon^2}}^{\qfun\pa{\slocal,h',0}}.\]
If $\slocal_{h'}$ is a \maxn node, then there are two sub-cases, depending on whether $\slocal$ is $h'$-decidable or $h'$-undecidable. 
If $\slocal_{h'}$ is $\slocal$-decidable this means that at the end of the while loop of the \maxn node, only the path towards $\slocal$ will remain in $\mathcal{L}$ because we are on the event when all outputs are within their confidence intervals. For any call with parameters $\pa{\klocal, \elocal}$ done to $\slocal_{h'}$ a call with with the same parameters in done to $\slocal_{h'+1}$. From this, we get that 
\[R_{\slocal_{h'+1}} = R_{\slocal'_{h'}}\]
Because $\slocal$ is $h'$-decidable, we have that $\qfun\pa{\slocal,h'+1} = \qfun\pa{\slocal,h'}$. Furthermore, since $\slocal_{h'}$ is a \maxn node, we have that $\sP{\slocal_{h'+1}|\slocal_{h'}} = 1$. We combine this with our induction assumption to get
\begin{align*}
\EE{R_\slocal} &\le \EE{R_{\slocal_{h'+1}}}\sP{\slocal|\slocal_{h'+1}}\pa{\cfrac{18\ln(t/\delta)}{(1-\eta)^2(1-\gamma)^2\epsilon^2}}^{\qfun\pa{\slocal,h'+1}}\\
&= \EE{R_{\slocal_{h'}}}\sP{\slocal|\slocal_{h'}}\pa{\cfrac{18\ln(t/\delta)}{(1-\eta)^2(1-\gamma)^2\epsilon^2}}^{\qfun\pa{\slocal,h'}}.
\end{align*}
If $\slocal$ is $h'$-undecidable, this means that at the end of the loop $U_i^{\slocal, \tlocal} \le (1-\eta)\epsilon$. Also $\klocal_i$ only increase one by one in the loop therefore
\[\cfrac{16\ln(t/\delta)}{(1-\eta)^2(1-\gamma)^2\epsilon^2} \le  \klocal_i < 1+ \cfrac{16\ln(t/\delta)}{(1-\eta)^2(1-\gamma)^2\epsilon^2}\]
As $\delta \le 1$ and $\epsilon \le \frac{1}{1-\gamma}$ (otherwise the problem is trivial) we have for any $\tlocal\ge2$
\[\klocal_i < \cfrac{18\ln(t/\delta)}{(1-\eta)^2(1-\gamma)^2\epsilon^2}\cdot\]
From the above we deduce that 
\begin{align*}
\EE{R_{\slocal_{h'+1}}} &\le \PP{R_{\slocal_{h'+1}} > 0}\cfrac{18\ln(t/\delta)}{(1-\eta)^2(1-\gamma)^2\epsilon^2}
\le \EE{R_{\slocal_{h'+1}}}\cfrac{18\ln(t/\delta)}{(1-\eta)^2(1-\gamma)^2\epsilon^2}\cdot
\end{align*}
Furthermore, since $\slocal_{h'}$ is $\slocal$-undecidable, we have $\qfun\pa{\slocal,h',H}=\qfun\pa{\slocal,h'+1,H}+1$. Also, because $\slocal-{h'}$ is a \maxn node, we have $\sP{\slocal_{h'+1}\|\slocal_{h'}} = 1$. We combine this with our induction assumption to prove our induction property for $\slocal-{h'}$ and that finish the induction when $H=0$. We generalize for any $H$ with exactly the same reasoning over $H$.
Now, taking $h' = 0$ we have 
\begin{align*}
\EE{R_\slocal} &= \EE{R_{\slocal_{0}}}\sP{\slocal|\slocal\pa{h'}}\pa{\cfrac{18\ln(t/\delta)}{(1-\eta)^2(1-\gamma)^2\epsilon^2}}^{\qfun\pa{\slocal,0}} 
= m\sP{\slocal}\pa{\cfrac{18\ln(t/\delta)}{(1-\eta)^2(1-\gamma)^2\epsilon^2}}^{\qfun\pa{\slocal}}\cdot
\end{align*}
Let $R\pa{h}$ be the number of times any node of depth $h$ has been reached. We have that
 \begin{align*}
 \EE{R\pa{h}} &= \EE{\sum_{\slocal\in\mathcal{N}_h} R_\slocal} = \sum_{\slocal\in\mathcal{N}_h} \EE{R_\slocal}
  \le m\sum_{\slocal\in\mathcal{N}_h}\sP{\slocal}\pa{\cfrac{18\ln(t/\delta)}{(1-\eta)^2(1-\gamma)^2\epsilon^2}}^{\qfun\pa{\slocal}} \cdot
 \end{align*}
By the definitions of $\lambda$ and $C$  we have 
\[ \forall \freevar \quad \EEs{\slocal}{e^{\freevar\qfun\pa{\slocal}}} \le e^{\lambda \freevar}\hspace{5mm}\abs{\mathcal{N}_h} \le C.\]
By taking $e^\freevar = \cfrac{18\ln(t/\delta)}{(1-\eta)^2(1-\gamma)^2\epsilon^2}$ we have that
 \[\EE{R\pa{h}} \le \nalg\pa{\cfrac{18\ln(t/\delta)}{(1-\eta)^2(1-\gamma)^2\epsilon^2}}^{\lambda}\cdot\]
 We can finally bound the total number of calls $\tlocal$ to the generative model as 
 \begin{align*}
 \EE{T} &\le \sum_{h\le h_{\text{max}}} \EE{R\pa{h}}\\
 &\le h_{\text{max}}m\pa{\cfrac{18\ln(t/\delta)}{(1-\eta)^2(1-\gamma)^2\epsilon^2}}^{\lambda}\\
 &\le 2m\frac{\ln\pa{1/\epsilon} + \ln\pa{1/\pa{1-\gamma}}}{\ln\pa{\eta/\gamma}}\pa{\cfrac{18\ln(t/\delta)}{(1-\eta)^2(1-\gamma)^2\epsilon^2}}^{\lambda}\\
 &\le 2\cfrac{\ln(1/\delta)}{(1-\gamma)^2\epsilon^2}\frac{\ln\pa{1/\epsilon} + \ln\pa{1/\pa{1-\gamma}}}{\ln\pa{\eta/\gamma}}\pa{\cfrac{18\ln(t/\delta)}{(1-\eta)^2(1-\gamma)^2\epsilon^2}}^{\lambda}\cdot\\ 
 \end{align*}
Setting $\eta = \gamma^{\max\pa{1,\frac{1}{\ln(1/\epsilon)}}}$ gives the bound. 
\end{proof}

\subsection{Proof of Theorem~\ref{thm:ncalls}}
\begin{proof}
Any near-optimal node of depth $h$ generates at most $\frac{18\ln(t/\delta)}{(1-\eta)^2(1-\gamma)^2\epsilon^2}$ samples. Therefore the maximum number of samples is
\begin{align*}
n &\le \sum_{h\le h_{\text{max}}}N^h\kappa^h\cfrac{18\ln(t/\delta)}{(1-\eta)^2(1-\gamma)^2\epsilon^2}\\
&=\cO\pa{h_{\text{max}}(N\kappa)^{h_{\text{max}}}\cfrac{\ln(t/\delta)}{(1-\eta)^2\epsilon^2}}\\
&=\cO\pa{h_{\text{max}}(1/\epsilon)^{\ln(N\kappa)/\ln(\eta/\gamma)}\cfrac{\ln(t/\delta)}{(1-\eta)^2\epsilon^2}}\cdot
\end{align*}
Setting $\eta = \gamma^{\max\pa{1,\frac{1}{\ln(1/\epsilon)}}}$ gives the bound. 
\end{proof}

\end{document}